%% file: main_arxiv_v4.tex
\pdfoutput=1  
\documentclass[11pt]{article}

\PassOptionsToPackage{numbers, compress}{natbib}

\usepackage[utf8]{inputenc}
\usepackage[T1]{fontenc}
\usepackage[margin=1in]{geometry}
\usepackage{amsmath,amssymb,amsfonts,amsthm}
\usepackage{booktabs}
\usepackage{array}
\usepackage{tabularx}
\usepackage{multirow}
\usepackage{microtype}
\usepackage{graphicx}
\usepackage{xcolor}
\usepackage{paralist}
\usepackage{enumitem}
\usepackage{url}
\usepackage{hyperref}
\usepackage{nicefrac}
\usepackage{mathtools}
\usepackage{caption}
\usepackage{authblk}
\usepackage{natbib}

\hypersetup{colorlinks=true, linkcolor=blue, citecolor=blue, urlcolor=blue}

\theoremstyle{plain}
\newtheorem{theorem}{Theorem}
\newtheorem{proposition}{Proposition}
\newtheorem{lemma}{Lemma}
\theoremstyle{definition}

\theoremstyle{remark}

\newcolumntype{P}[1]{>{\centering\arraybackslash}p{#1}}

\usepackage{tikz}
\usetikzlibrary{positioning, arrows.meta, calc, fit, backgrounds}

\definecolor{cblue}{rgb}{0.88,0.92,1.00}
\definecolor{camber}{rgb}{1.00,0.93,0.83}
\definecolor{cgreen}{rgb}{0.86,0.97,0.86}
\definecolor{cpurple}{rgb}{0.93,0.88,1.00}
\definecolor{dblue}{rgb}{0.25,0.40,0.75}
\definecolor{damber}{rgb}{0.80,0.55,0.15}
\definecolor{dgreen}{rgb}{0.25,0.55,0.30}
\definecolor{dpurple}{rgb}{0.45,0.35,0.70}

\newcommand{\R}{\mathbb{R}}
\newcommand{\E}{\mathbb{E}}
\newcommand{\Var}{\operatorname{Var}}

\newcommand{\norm}[1]{\left\lVert #1 \right\rVert}
\newcommand{\ip}[2]{\left\langle #1,#2\right\rangle}
\newcommand{\Span}{\operatorname{span}}

\newcommand{\task}[1]{\texttt{#1}}
\DeclareMathOperator{\Krylov}{\mathcal{K}}

\title{Recoverable but Not Stationary:\\Local Linear Structures in Weights and Activations}

\author[${}$]{Irina Piontkovskaia}

\author[1,2]{Sergey Nikolenko}

\affil[1]{St. Petersburg Department of the Steklov Institute of Mathematics, St. Petersburg, Russia, \texttt{sergey@logic.pdmi.ras.ru}}
\affil[2]{St. Petersburg State University}

\date{\today}

\begin{document}

\maketitle

\begin{abstract}
Task vectors, LoRA, activation steering, and random search around pretrained weights all suggest that learned behaviour can be controlled by linear directions. We ask which linear structures actually exist and on what scale. In a synthetic multitask transformer and LoRA adapters on DistilGPT-2 / GPT-2 we find strong \emph{local} low-rank task-gradient structure but reject the fixed-task-plane hypothesis: static bases miss the recovery direction, and the useful basis drifts substantially within $100$ steps. However, the first recovery updates form a \emph{trajectory-prefix} basis capturing $77\%$ of the LoRA recovery displacement. We develop random search theory with a Gaussian local-linear theorem that justifies the effectiveness of random parameter search even in very high dimensions. We also study the relation between parameter perturbations and activation steering: a single gradient step produces an activation shift with $0.58$ cosine to a labelled-contrast CAA steering vector, with a similar steering effect on Qwen-$0.5$B BoolQ statements. We validate our results with experiments on synthetic Transformers and LLMs. Our results suggest that linear structures in trained networks are not global task directions, but evolving local geometries that partially persist across parameter and activation spaces.
\end{abstract}

\section{Introduction}
\label{sec:intro}

Much of recent work on editing models after training rests on one geometric assumption: that the behaviour a network has learned for a task can be pinned to a single linear direction, either in weight space or in activation space, and that moving along that direction lets us add, remove, or amplify the behaviour without retraining. Approaches such as task vectors~\citep{ilharco2023editing}, model soups~\citep{wortsman2022modelsoups}, and TIES-merging~\citep{yadav2023tiesmerging} treat fine-tuning updates as vectors that can be added and subtracted in parameter space. LoRA~\citep{hu2022lora} restricts adaptation to a low-rank slice of the weight matrices. Activation addition~\citep{turner2023actadd}, function vectors~\citep{todd2024functionvectors}, representation engineering~\citep{zou2023repe}, and ReFT~\citep{wu2024reft} move along linear directions in hidden-state space instead. These methods look different on the surface, but they share the same underlying claim: the relevant structure is low-dimensional and locally linear. A natural question arises: what kind of linear structure actually exists in a trained network? Where do they live, and how far one can move along before they stop being linear?

A more provocative version of the same question comes from recent work on \emph{random search around pretrained weights}. \citet{gan2026neural} argue that a sufficiently pretrained model is surrounded by a dense ``thicket'' of task-improving directions that can be found simply by sampling random parameter perturbations, keeping the ones that help, and combining them. This is a counterintuitive result: if useful directions were rare needles in a $10^9$-dimensional haystack, random sampling should never find them.

The starting point for this work was that random search apparently \emph{does} find useful directions. Naturally, random search is not an optimizer competing with gradient descent, gradient descent always wins on raw accuracy. But we can use it as a {measuring tool}: if random perturbations drawn inside one subspace reliably beat perturbations drawn inside another, at the same step size, that can tell us which subspace is better aligned with the task.

To study these subspaces in practice, we first run experiments on synthetic models in a \emph{recovery after forgetting} regime. To study local geometry, we need a known, working displacement to compare candidate subspaces against. We obtain one as follows: first train a small network on a mixture of tasks, reaching a \emph{multitask checkpoint} $\theta_{\rm mt}$, and then deliberately overfitting it to a single task, running single-task training long enough that one of the other tasks degrades; we call the result the \emph{forgotten checkpoint} $\theta_{\rm forg}$. Finally, we retrain briefly on the original mixture, which brings the lost task back; call this $\theta_{\rm rec}$. As a result, the vector $\Delta_{\rm GD}=\theta_{\rm rec}-\theta_{\rm forg}$ becomes a task-specific adaptation that gradient descent has produced, and we can then study small subspaces in terms of how much of $\Delta_{\rm GD}$ they contain. In this work, ``recovery'' always refers to this loop. Note that nothing in our theoretical results is specific to forgetting, and the same results could apply to any small gradient-based adaptation around a checkpoint, but the forgetting setup has a clean target $\Delta_{\rm GD}$ to measure against.

We use three ways to score random search results in this setup. We draw $N$ candidate perturbations, evaluate them on a held-out split, and keep the best. \emph{Best-of-$N$} reports the single best candidate's gain, and this is the quantity used in our theoretical results. \emph{Ensemble-$k$-of-$N$} combines the outputs of the top $k$ candidates by majority vote and reports the ensemble's gain. \emph{Pass@$N$} reports the fraction of test items solved by {at least one} kept candidate, which shows an upper bound that a perfect selector could reach.

\begin{figure}[!t]
\centering
\setlength{\tabcolsep}{3pt}\footnotesize
\begin{tabular}{P{.3\linewidth}P{.32\linewidth}P{.32\linewidth}}
\multicolumn{3}{c}{\includegraphics[width=\linewidth]{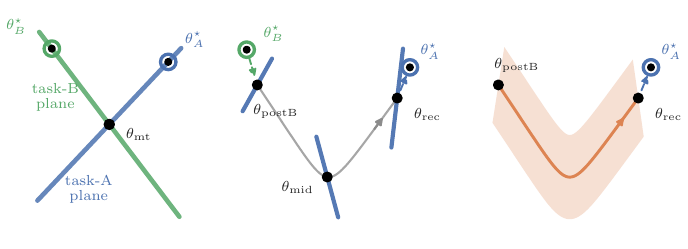}} \\[-2pt]
{(a) Static-plane hypothesis} & {(b) Observed: the local plane drifts} & {(c) Proposed: a moving trajectory bundle} \\
\end{tabular}
\caption{From a static task plane to a moving trajectory: (a)~the natural hypothesis is that each task corresponds to a fixed subspace through $\theta_{\rm mt}$ and recovery moves inside it; (b)~we observe that the local task plane drifts substantially along the recovery path; (c)~we propose a low-dimensional bundle that follows the trajectory whose tangent at each point lies in the local Krylov subspace.}
\label{fig:schematic}
\end{figure}

We find that the recovery direction is real, low-dimensional, and useful, but it is \emph{not} contained in any fixed task plane. The task-gradient subspace measured at the forgotten checkpoint captures only a modest fraction of the recovery displacement, and the subspace measured before forgetting captures almost none; moreover, the subspace itself drifts as recovery proceeds. What \emph{does} work is the trajectory itself: a basis built from the first several recovery steps captures most of the displacement and is the best subspace to search inside. So the right object is not a static plane through a checkpoint but a short, moving segment of the optimization path: low-dimensional but not stationary. At LLM scale, we find the same qualitative pattern with important caveats: the low-rank reading remains clean on our small controlled models and on LoRA adapters but is only partially confirmed on billion-parameter models, where naive rank estimates are hard to separate from sampling noise. Figure~\ref{fig:schematic} illustrates this main idea.

The most counterintuitive result is that random search works at all in a $10^9$-dimensional space, and we explain it with a theoretical result. When we sample an isotropic Gaussian perturbation, the amount of a fixed useful direction $a$ that it picks up has variance $\sigma^2\norm{a}^2$, which does not depend on the number of parameters at all. Best-of-$N$ then amplifies this by a $\sqrt{\log N}$ factor, so the target signal grows with the search budget and is independent of dimension, while damage to an unrelated, orthogonal direction stays flat in $N$. Thus, a budgeted random search helps the targeted task much more than it hurts the others; below, we formalize this result in a theorem. Note that this only holds when the loss is roughly linear in the perturbation, and we locate that regime empirically (e.g., around a perturbation scale of $10^{-4}$ on Qwen2.5-0.5B).

Weight edits and activation steering are usually treated as separate tools. There is a connection between them: a weight perturbation $\delta\theta$ shifts each layer's activations by $D_\theta h_\ell\,\delta\theta$, the pushforward of the weight move into activation space. If $\delta\theta$ helped the task, its activation counterpart (``shadow'') is a candidate steering vector. We have tested $\delta\theta$ on a single gradient step on Qwen2.5-0.5B, and the resulting activation shift has cosine similarity around $0.58$ to a contrastive (true-vs-false) steering vector built independently, peaks at late decoder layers, and recovers accuracy when injected back. These results are still preliminary, but they show that the formal connection can be made empirical.

Finally, we note some limitations of our results. Although random search can, in some of our experiments, even outperform gradient descent on the target task, it is unstable and its behaviour on the other tasks is hard to guarantee, so it is in no way a replacement for gradient descent. We see our results instead as a starting point for a serious analysis of this phenomenon, one that may eventually yield a sound methodology and a clearer understanding of when it applies. The same holds for activation steering, where our results vary across layers and across the direction of the perturbation. Even so, our random search theory applies just as well to the search for steering directions, so unifying the two into a single account may be an interesting direction for future work.

The remainder of the paper is organized as follows. Section~\ref{sec:related} reviews related work, and Section~\ref{sec:setup} fixes notation and defines the geometric objects we use. Section~\ref{sec:theory} proves the three results the paper rests on: the best-of-$N$ theorem, a norm-matched companion explaining why padding a subspace hurts, and a short lemma showing that the recovery displacement lives in a Krylov subspace rather than along any fixed direction; it closes with the pushforward identity. Section~\ref{sec:synthetic_lora} runs the controlled synthetic Transformer; Section~\ref{sec:lora_validation} repeats the probes on LoRA adapters over DistilGPT-2 and GPT-2. Section~\ref{sec:steering} scales up to Qwen and OLMo: it asks whether the low-rank picture survives, shows that budgeted random search finds improvements in the regime the theorem predicts, and connects weight perturbations to activation steering. Section~\ref{sec:discussion} discusses limitations and Section~\ref{sec:conclusion} concludes the paper. Detailed per-run tables and the more diagnostic experiments are collected in the Appendix.

\section{Related Work}
\label{sec:related}

\paragraph*{Intrinsic dimension and weight-space connectivity.}
\citet{li2018intrinsic} train neural networks in low-dimensional random subspaces and find that many objectives can be solved with far fewer dimensions than the parameter count suggests. Mode connectivity work~\citep{garipov2018losssurface, draxler2018essentially} shows that independently trained networks can be connected by simple low-loss curves, exposing further low-dimensional structure in weight space. Loss landscape and ensemble studies~\citep{fort2019deep,li2018visualizing} extend the picture. We add a temporal axis to this story: the relevant low-dimensional structure for recovery after forgetting is not the static intrinsic basis but a moving trajectory-aligned bundle.

\paragraph*{Low-rank adaptation.}
LoRA~\citep{hu2022lora} freezes pretrained weights and learns rank-deficient adapter updates, motivated by the conjecture that adaptation has low intrinsic rank. Adapter-tuning and parameter-efficient methods~\citep{aghajanyan2021intrinsic,houlsby2019parameter} relate. Our LoRA experiments place the trajectory-bundle / Krylov-prefix story inside this standard framework: the LoRA recovery displacement \emph{is} low-rank, but its low-rank object is empirical Krylov, not the static task subspace at any single checkpoint.

\paragraph*{Task vectors and model merging.}
Task arithmetic~\citep{ilharco2023editing} treats fine-tuning deltas as composable vectors. Model souping~\citep{wortsman2022modelsoups} averages fine-tuned weights to improve accuracy and robustness. TIES-merging~\citep{yadav2023tiesmerging} resolves interference between task-specific deltas. These methods presuppose that fine-tuning deltas behave additively, but our results qualify this picture.

\paragraph*{Activation-space linear structures.}
Activation Addition~\citep{turner2023actadd} constructs steering vectors from contrastive activation differences. Function vectors~\citep{todd2024functionvectors} find compact hidden-state representations of in-context functions. Representation engineering~\citep{zou2023repe} frames activation-space directions as a general approach to monitoring and controlling high-level behaviour. ReFT/LoReFT~\citep{wu2024reft} learns low-rank interventions in representation space rather than adapting weights. We provide a partial bridge: weight-space perturbations sometimes have reusable activation-space counterparts.

\paragraph*{Multitask gradient geometry.}
PCGrad~\citep{yu2020gradient}, MGDA-style methods~\citep{sener2018multi}, GradNorm~\citep{chen2018gradnorm}, CAGrad~\citep{liu2021cagrad} explicitly reason about per-task gradient interactions. Catastrophic forgetting and continual learning regularizers~\citep{kirkpatrick2017overcoming, lopezpaz2017gradient} bound how much a current update may move along old task subspaces. Instead, we do temporal diagnostics, asking how task-gradient subspaces {evolve} through pretraining, single-task continuation, forgetting, and recovery.

\paragraph*{Random search and neural thickets.}
\citet{gan2026neural} argue that sufficiently pretrained large models are surrounded by dense and diverse task-improving specialists, recoverable by random parameter perturbation, top-$K$ selection, and ensembling. Earlier evolutionary strategy and zeroth-order optimization work~\citep{malladi2023mezo,salimans2017evolution} established that random search at the right scale can move LLMs in useful directions. We treat random search not as an optimizer to compete with gradients but as a probe of local geometry: the \emph{order} of proposal-family curves at matched budget is informative about the underlying anisotropy.

All these prior directions have the same linear structure story: low-dimensional directions exist and are relatively easy to find. Our contribution is to ask which of those structures are local vs.\ global, static vs.\ moving, weight-space vs.\ activation-space, and to show empirically that the answer depends on coordinate system and intervention scale.

\section{Notation and definitions}
\label{sec:setup}

In this section, we introduce the definitions and notation that the rest of the paper relies on. A few of these objects are not standard, so we define all of them here and use throughout the paper.

\paragraph{Parameter space and loss.}
Let $\theta\in\R^D$ be the (possibly adapter-restricted) parameters of a network, and let a \emph{task} $\tau$ have loss $L_\tau(\theta)$. We write $g_\tau(\theta)=\nabla_\theta L_\tau(\theta)$ and drop $\tau$ when the task is clear. Near a checkpoint $\theta$ we consider the first order approximation
$$L(\theta+\delta)-L(\theta)\approx a^\top\delta,$$ where $a=-\nabla_\theta L(\theta)$ is the direction that most improves the task; we call $a$ the \emph{useful direction}. Empirically $a$ lives in a low-dimensional subspace, but the algorithms we study never know that subspace.

\paragraph{Recovery after forgetting.}
Our experimental loop has three phases:
\begin{enumerate}[label=(\arabic*)]
  \item \emph{multitask training} to a checkpoint $\theta_{\rm mt}$ on an equal mixture of $T$ tasks;
  \item \emph{single-task continuation (forgetting):} continue training on one task only, reaching $\theta_{\rm forg}$, at which one or more of the other tasks has degraded;
  \item \emph{recovery:} from $\theta_{\rm forg}$, take a small number of gradient steps on the original mixture (we use $100$ for the synthetic model, $50$ for LoRA), reaching $\theta_{\rm rec}$.
\end{enumerate}

The \emph{recovery displacement} is defined as
\[
\Delta_{\rm GD}\;=\;\theta_{\rm rec}-\theta_{\rm forg},
\]
i.e., the ground-truth gradient-based adaptation against which candidate subspaces are compared.

\paragraph{Candidate subspaces.}
We distinguish the \emph{useful direction} $a$ (a property of the loss, which the sampler does not know) from a \emph{candidate subspace} $U\subseteq\R^D$ of dimension $r$ that we choose, in order to test how well it aligns with $a$ or with $\Delta_{\rm GD}$. We use three families of candidate subspaces:
\begin{itemize}
\item \emph{local task subspace} $U_\tau^{(r)}(\theta)$: the span of the top-$r$ eigenvectors of the per-task gradient covariance $\tfrac1m\sum_i g_\tau(\theta;b_i)g_\tau(\theta;b_i)^\top$ estimated from $m$ minibatches at checkpoint $\theta$; we write \emph{Task (local)} for $\theta=\theta_{\rm forg}$ and \emph{Task (pre-int.)} (before interference) for $\theta=\theta_{\rm mt}$; this is the natural ``task plane'';
\item \emph{trajectory prefix} $U_{\rm traj}^{(k)}$: the span of the first $k$ recovery step vectors $\theta_1-\theta_0,\dots,\theta_k-\theta_{k-1}$ along the recovery path; by construction, $\Delta_{\rm GD}$ lies in the full trajectory span, and a short prefix may capture most of it;
\item \emph{explicit Krylov} $\Span\{g,Hg,H^2g,H^3g\}$ at $\theta_{\rm forg}$, built from finite-difference Hessian-vector products; we use Krylov subspaces as a curvature diagnostic, not as a working basis.
\end{itemize}

\paragraph{Random search and its three scores.}
At a checkpoint $\theta$ with task score $S=-L_\tau$, we sample $N$ perturbations, evaluate $S(\theta+\delta_i)$ on a selection split, and keep the top $k$. The \emph{native} sampler is isotropic Gaussian, $\delta\sim\mathcal N(0,\sigma^2 I)$. As a diagnostic we also restrict sampling to a candidate subspace $U$, either Gaussian in $U$, $\delta\sim\mathcal N(0,\sigma^2 P_U)$, or norm-matched on the sphere of $U$. By comparing scores across choices of $U$ we can rank the subspaces. 

We score the result in three ways, defined here:
\begin{itemize}
\item \emph{best-of-$N$}: the gain of the single best candidate; this is the quantity Theorem~\ref{thm:gaussian_search} describes;
\item \emph{ens-$k$-of-$N$}: the gain of the majority-vote ensemble of the top $k$ candidates;
\item \emph{pass@$N$}: the fraction of report-split items solved by at least one kept candidate (an upper envelope).
\end{itemize}

\paragraph{Activation steering and CAA.}
Write $h_\ell(x;\theta)\in\R^d$ for the residual-stream activation at layer $\ell$. \emph{Activation steering} replaces $h_\ell$ by $h_\ell+\alpha v_\ell$ for a \emph{steering vector} $v_\ell$ and strength $\alpha$. The classical construction is \emph{centroid-based activation addition} (CAA): take labelled paired prompts (e.g.\ true vs.\ false statements) and let $v_\ell$ be the mean activation difference between the two classes. The pushforward identity of Section~\ref{sec:theory:pushforward} gives an alternative: $v_\ell=\E_x[D_\theta h_\ell(x;\theta)\,\delta\theta]$, the example-averaged activation shift induced by a weight perturbation $\delta\theta$ (a gradient step, or a selected random perturbation).

\section{Theory: Local Linear Structure}
\label{sec:theory}

In this section, we prove theoretical results that the experiments then validate and use. First, a best-of-$N$ theorem that explains why isotropic random search picks up useful signal regardless of dimension (Theorem~\ref{thm:gaussian_search}). Second, a norm-matched companion that explains why, when comparing candidate subspaces at a fixed step size, a smaller well-aligned subspace beats a larger one that merely contains the answer (Proposition~\ref{prop:aligned_search}). Third, a short lemma showing the recovery displacement lies in a Krylov subspace, so no single fixed direction can capture it (Lemma~\ref{lem:krylov}). We close with the parameter-to-activation pushforward (Section~\ref{sec:theory:pushforward}).

There is one important notational remark to keep in mind throughout, because it might be the source of confusion: $a$ is the \emph{useful direction}, a fixed property of the loss that the sampler does not know; $U$ is a \emph{candidate subspace} that {we} ourselves choose in order to test it. The native search samples isotropically and never references $U$; the diagnostic search samples inside a chosen $U$. The theorem is about the native sampler and a fixed $a$; the proposition is about the diagnostic sampler and a chosen $U$.

\subsection{Best-of-\texorpdfstring{$N$}{N} random search picks up dimension-independent signal}
\label{sec:theory:gaussian}

\begin{theorem}[Gaussian best-of-$N$ in a local-linear neighbourhood]
\label{thm:gaussian_search}
Consider a checkpoint $\theta$ and its neighborhood where the score $S=-L_\tau$ satisfies $S(\theta+\delta)-S(\theta)=a^\top\delta+\mathcal O(\norm\delta^2)$ for a fixed useful direction $a\in\R^D$. Sample $\delta_1,\dots,\delta_N\sim\mathcal N(0,\sigma^2 I)$ and select $\delta_{n^\star}=\arg\max_n a^\top\delta_n$. Then:
\begin{enumerate}[label=(\arabic*)]
\item each projection $a^\top\delta_n$ is centred Gaussian with variance $\sigma^2\norm{a}^2$, \emph{independent of the ambient dimension $D$};
\item the expected best-of-$N$ gain is $\E[\max_{n\le N}a^\top\delta_n]=\sigma\norm{a}\sqrt{2\log N}\,(1+o(1))$, growing in the budget but not in $D$;
\item for any $b\perp a$ (an unrelated side-task direction), the selected sample has $\E[b^\top\delta_{n^\star}]=0$ and typical magnitude $|b^\top\delta_{n^\star}|=\mathcal O(\sigma\norm b)$, \emph{independent of $N$}.
\end{enumerate}
\end{theorem}

\begin{proof}
For $\delta\sim\mathcal N(0,\sigma^2 I)$, $a^\top\delta$ is centred Gaussian with $\Var(a^\top\delta)=\sigma^2\norm a^2$, giving part 1. Writing $X_n=a^\top\delta_n/(\sigma\norm a)\sim\mathcal N(0,1)$ i.i.d., the classical extreme-value result $\E\max_{n\le N}X_n=\sqrt{2\log N}(1+o(1))$ gives part 2. For part 3, $(a^\top\delta,b^\top\delta)$ is jointly Gaussian with covariance $\sigma^2\ip a b$; if $\ip a b=0$ the two are independent, so conditioning on the selection event (a function of $a^\top\delta$ only) leaves $b^\top\delta_{n^\star}\sim\mathcal N(0,\sigma^2\norm b^2)$. Sub-Gaussian concentration gives $|b^\top\delta_{n^\star}|=\Theta(\sigma\norm b)$, independent of $N$. (For weakly correlated $a,b$ with $\rho=\ip a b/(\norm a\norm b)$ the side movement gains a $|\rho|\sqrt{\log N}$ term, still flat once $|\rho|\sqrt{\log N}\ll1$.) 
\end{proof}

The first part of Theorem~\ref{thm:gaussian_search} is the main result: even though the sampler spreads its variance over all $D$ coordinates, the amount of the useful direction $a$ in a single draw has variance $\sigma^2\norm a^2$, which does not shrink with $D$. This means that random search is not hopeless at LLM scale: the per-sample signal is the same for all values of $D$, even very large ones. The second part says that best-of-$N$ amplifies that signal by the usual Gaussian $\sqrt{2\log N}$. The third part is the asymmetry we care about: selecting for $a$ does not, on average, move an orthogonal direction $b$ at all, and even the typical accidental movement stays flat as the budget grows. Target gains grow with $N$, while collateral damage does not. 

This is somewhat counterintuitive since a random vector is nearly orthogonal to $a$ in high dimensions, so it naively looks like finding an aligned one should take exponentially many tries. That is indeed true of the {angle}, which concentrates near $90^\circ$. But the angle is the wrong quantity: the projection $a^\top\delta$ is the product of a vanishing cosine and a $\sqrt D$-growing length, and the two cancel to leave a $D$-independent $\Theta(\sigma\norm a)$. Best-of-$N$ then extracts $\Theta(\sigma\norm a\sqrt{\log N})$ regardless of $D$.

If the useful directions span a subspace of dimension $r_\star>1$ rather than a single $a$, the same statement holds coordinate-by-coordinate: the per-sample signal is governed by $\norm{P_\star a}$ where $P_\star$ projects onto the useful subspace, and again does not involve $D$. We state the theorem for a single $a$ for clarity; the experiments measure effective ranks $r_\star$ between $2$ and $\sim$$12$ in models from $\sim$$5\!\cdot\!10^5$ to $\sim$$5\!\cdot\!10^8$ parameters.

\subsection{Why a smaller, well-aligned subspace beats a larger one}
\label{sec:theory:normmatched}

Theorem~\ref{thm:gaussian_search} might suggest that, when comparing candidate subspaces, the one that contains the most of $a$ should win. However, it turns out to be wrong at fixed step size: padding a candidate subspace $U$ with directions orthogonal to $a$ leaves the contained signal unchanged but dilutes it across more dimensions.

\begin{proposition}[Matched-norm search prefers density, not mass]
\label{prop:aligned_search}
Restrict sampling to a candidate subspace $U$ of dimension $r$ by drawing $\delta=\rho u$ with $u$ uniform on the unit sphere of $U$ (fixed $\rho$). Then
$$\E[a^\top\delta]=0\qquad\text{and}\qquad\Var(a^\top\delta)=\rho^2\norm{P_U a}^2/r.$$
Therefore, the per-dimension search signal scales as $\norm{P_U a}/\sqrt r$: adding to $U$ directions orthogonal to $a$ raises $r$ without raising $\norm{P_U a}$, and so {lowers} the signal.
\end{proposition}

\begin{proof}
For $u$ uniform on the unit sphere of $U$, rotational invariance inside $U$ gives $\E[uu^\top]=(1/r)P_U$, so $\E[(a^\top u)^2]=a^\top\E[uu^\top]a=\norm{P_U a}^2/r$. With $\delta=\rho u$, $\Var(a^\top\delta)=\rho^2\norm{P_U a}^2/r$. Adding directions orthogonal to $a$ leaves $\norm{P_U a}^2$ fixed and increases $r$, lowering $\norm{P_U a}/\sqrt r$.
\end{proof}

Proposition~\ref{prop:aligned_search} means that at fixed step size, the relevant quantity is signal \emph{density} $\norm{P_U a}/\sqrt r$ rather than signal \emph{mass} $\norm{P_U a}^2$. A rank-$32$ subspace that contains $a$ entirely has density $\norm a/\sqrt{32}$, which is worse than a rank-$10$ subspace containing $77\%$ of $a$. This is exactly the ordering we will see in the experiments, and, importantly, it is a prediction for \emph{best-of-$N$}, the metric the theorem describes; ensembling behaves slightly differently, as we discuss in Section~\ref{sec:lora_validation}.

\subsection{The recovery displacement is Krylov, not a fixed direction}
\label{sec:theory:krylov}

Theorem~\ref{thm:gaussian_search} and Proposition~\ref{prop:aligned_search} deal with the properties of {searching} once we know which way to look. The next lemma describes what the recovery direction actually {is}.

\begin{lemma}[Krylov recovery]
\label{lem:krylov}
Suppose $L_\tau$ is quadratic with constant Hessian $H$ over the recovery window, and let $\theta_{t+1}=\theta_t-\eta\nabla L_\tau(\theta_t)$. Then $g_t=(I-\eta H)^t g_0$ and
\[
\theta_T-\theta_0=-\eta\sum_{t=0}^{T-1}g_t\;\in\;\Krylov_T(H,g_0):=\Span\{g_0,Hg_0,\dots,H^{T-1}g_0\}.
\]
A fixed direction $u$ is invariant under one step iff $Hu\parallel u$, and a subspace $U$ is invariant iff $HU\subseteq U$.
\end{lemma}

\begin{proof}
With $L_\tau(\theta)=\tfrac12(\theta-\theta^\star)^\top H(\theta-\theta^\star)$, $g_t=H(\theta_t-\theta^\star)$ and $\theta_{t+1}-\theta^\star=(I-\eta H)(\theta_t-\theta^\star)$, so $g_{t+1}=(I-\eta H)g_t$ and $g_t=(I-\eta H)^t g_0$. The displacement $-\eta\sum_{t<T}g_t$ lies in $\Span\{g_0,\dots,H^{T-1}g_0\}$. Invariance of $u$ under $u\mapsto(I-\eta H)u$ requires $Hu\parallel u$; for a subspace, $HU\subseteq U$.
\end{proof}

Lemma~\ref{lem:krylov} means that recovery displacement lives in a Krylov subspace of dimension at most $T$ (low-dimensional), but a {single fixed direction} captures it only in the degenerate case where the gradient is an eigenvector of the Hessian, and a fixed subspace captures it only if the Hessian maps it into itself. Neither holds for a deep network at a forgotten checkpoint. So the recovery object is low-dimensional but \emph{moving}: higher-order terms $H^kg$ keep introducing new components as the trajectory unrolls.

Note also that the clean statement assumes a \emph{constant} Hessian. For a nonlinear loss the Hessian $H_t$ at successive iterates differs, and the displacement lies in a \emph{generalized} Krylov span $\Span\{g_0,H_0g_0,H_1H_0g_0,\dots\}$ rather than the powers of a single $H$. The single-$H$ lemma is therefore exact for a quadratic and holds to leading order only when the update is small enough that $H_t\approx H$ across the window, which is the correct regime in our short, small-norm recovery with $T\le100$ steps. We use the lemma to explain {why} a fixed direction is too rigid, but it does not let us reconstruct the path. Instead, we rely on the directly observed trajectory prefix, which reliably captures most of the recovery displacement (about $77\%$ on the LoRA adapters; Figure~\ref{fig:projection_mass} and Table~\ref{tab:projected}), and use it as our working proxy for the right Krylov subspace. The explicit finite-difference Krylov basis $\{g, Hg, H^2g, H^3g\}$ enters only as a curvature diagnostic (Appendix~\ref{app:krylov}).

\subsection{Parameter-to-activation pushforward}
\label{sec:theory:pushforward}

To connect weight perturbations to activation steering we need one identity. For an activation $h_\ell(x;\theta)$ and a weight perturbation $\delta\theta$,
\begin{equation}
h_\ell(x;\theta+\delta\theta)-h_\ell(x;\theta)\;\approx\;D_\theta h_\ell(x;\theta)\,\delta\theta,
\label{eq:pushforward}
\end{equation}
the pushforward of the weight move into activation space. If $\delta\theta$ improves the task, then its example-averaged activation shadow $v_\ell=\E_x[D_\theta h_\ell(x;\theta)\,\delta\theta]$ is a candidate steering vector with the same first-order effect on the loss. 

Whether that shadow still helps when reinjected into the clean model on \emph{independent} examples is an empirical question; we answer it for one gradient-step case (tight match) and one random-search case (partial transport) in Section~\ref{sec:steering}, where the effect is concentrated at the late decoder layers (around $\ell=19$--$20$ on Qwen2.5-0.5B).

\section{Experiments}
\label{sec:experiments}

In this section, we present our experimental evaluation on three scales: a controlled synthetic Transformer, where we can see the whole geometry directly (\S\ref{sec:synthetic_lora}), LoRA adapters on two real pretrained backbones (\S\ref{sec:lora_validation}), and billion-parameter LLMs (\S\ref{sec:steering}). 

\subsection{Synthetic Transformer: local structure, but no fixed plane}
\label{sec:synthetic_lora}

We begin with a setting we can fully control: a $4$-layer, $128$-dimensional, $2$-head causal Transformer ($\sim$$5\!\cdot\!10^5$ parameters) trained on four procedurally generated digit-sequence tasks, namely \task{copy}, \task{reverse}, \task{sort}, and \task{mod-add}. Multitask pretraining runs for $20\,000$ steps; single-task continuation runs for an additional $10\,000$ steps to induce forgetting; recovery runs up to $100$ steps on the original mixture. We use three seeds throughout. Table~\ref{tab:synth_summary} shows the geometry and recovery numbers and Table~\ref{tab:synth_search} the random-search ordering. We can draw the following conclusions.

\begin{table}[!t]
\centering\footnotesize
\begin{minipage}[t]{0.52\linewidth}
\centering
\caption{Synthetic Transformer: local geometry, forgetting, recovery; $n=3$ seeds, \task{copy}$\to$\task{sort} recovery.}
\label{tab:synth_summary}
\setlength{\tabcolsep}{4pt}
\begin{tabular}{lc}
\toprule
\textbf{Quantity} & \textbf{Value} \\
\midrule
Per-task gradient eff.\ rank $r_{90}$ & $2$--$8$ \\
Cross-task overlap (top-$r$, $6$ pairs) & $0.11$ \,[$0.08$, $0.13$] \\
\quad random-subspace baseline & $\sim$$6\!\cdot\!10^{-5}$ \\
Worst other-task EM drop, forgetting & $0.56$--$0.95$ \\
Recovery EM, full descent & $\mathbf{0.30}$ \\
\quad projected, local task plane & $0.10$ \\
\quad projected, pre-int.\ plane & $0.07$ \\
Local-plane drift, step $0\!\to\!100$ & $1.0 \to 0.12$--$0.17$ \\
\bottomrule
\end{tabular}
\end{minipage}\hfill
\begin{minipage}[t]{0.45\linewidth}
\centering
\caption{Synthetic random search on the forgotten \task{sort} task: best-of-$N$ exact-match and gain over the forgotten checkpoint ($n=3$ seeds).}
\label{tab:synth_search}
\setlength{\tabcolsep}{4pt}
\begin{tabular}{lcc}
\toprule
Subspace & EM & gain \\
\midrule
Trajectory prefix      & $\mathbf{0.152}$ & $\mathbf{0.090}$ \\
Task plane (local)     & $0.120$ & $0.059$ \\
Task plane (pre-int.)  & $0.106$ & $0.032$ \\
Isotropic              & $0.055$ & $0.003$ \\
Random rank-matched    & $0.045$ & $0.002$ \\
\bottomrule
\end{tabular}
\end{minipage}
\end{table}

\paragraph{Task gradients are low-rank, and the tasks are not orthogonal.}
\label{sec:synthetic:lowrank}
At the multitask checkpoint, each task's gradient covariance is low-rank, with effective rank between $2$ and $8$, a tiny fraction of the parameters (Table~\ref{tab:synth_summary}, row~1). The tasks are also not randomly oriented relative to each other: the weighted overlap between their top-$r$ subspaces averages $0.11$ across the six task pairs, four orders of magnitude above the random-subspace baseline. So there is genuine low-dimensional task structure, and it is shared across tasks more than chance would predict.

\paragraph{Low-rank structure does not prevent forgetting.}
Low-rank, nearly-orthogonal task subspaces might suggest that the tasks should not interfere. In practice, we see that they do anyway: after single-task continuation, the worst other-task accuracy drops by between $0.56$ and $0.95$ (Table~\ref{tab:synth_summary}), even for pairs whose subspaces overlap below $0.10$. Even if two tasks live in nearly orthogonal subspaces, this fact does not protect one from catastrophic interference by the other.

\paragraph{The recovery subspace moves, and a static plane cannot recover the task.}
\label{sec:synthetic:static}
This is the central negative result. We compare three ways of running recovery at matched step size: full gradient descent, descent with the gradient \emph{re-projected at every step} onto the local task subspace at the forgotten checkpoint, and the same projected onto the pre-forgetting (clean) task subspace. Table~\ref{tab:synth_summary} shows that full recovery roughly quadruples the exact-match accuracy that either static-plane projection reaches: the task plane measured at the very checkpoint where the task was learned captures less than a third of what unconstrained recovery achieves. 

Note that in this experiment, ``projected recovery'' re-projects the gradient at \emph{every} step; this is not the same as projecting the final unconstrained displacement, because constraining the path changes the path. Moreover, the task plane itself moves: the top-$r$ subspace at step $t$ has principal-angle overlap with the step-$0$ subspace that falls from $1.0$ to about $0.12$--$0.17$ over $100$ steps, while staying well above random sampling. Here both subspaces are top-$r$ subspaces for the same fixed $r$ (by construction), so this is a drift in \emph{orientation} --- where the $r$-dimensional plane points --- not a change in its dimension.
The two facts are consistent: the static plane misses recovery (performance) precisely because the relevant plane is moving (geometry).

\paragraph{Random search prefers the trajectory, exactly as the theory predicts.}
\label{sec:synthetic:randopt}
Finally, we use random search as a probe. Drawing norm-matched perturbations inside each candidate subspace on the forgotten \task{sort} task, the ordering is clean and reproduces across seeds (Table~\ref{tab:synth_search}): searching inside the first $\sim$$10$ recovery steps beats searching inside the local task plane, which is better than the pre-forgetting plane, and all of them are far better than isotropic and random controls.

There are two details here that follow from Proposition~\ref{prop:aligned_search}. First, the trajectory prefix wins even though its rank is smaller than the task planes'. Second, concatenating the task plane onto the trajectory (rank $48$) does \emph{not} help despite containing it, because the extra directions are mostly orthogonal to the recovery signal and dilute it. The ensemble and pass@$N$ variants follow the same ordering (Appendix~\ref{app:synthetic}).

\begin{table}[!t]
\centering
\footnotesize
\caption{Per-layer activation pushforward of trajectory-family parameter deltas on the synthetic Transformer, averaged across seeds at recovery step $T=100$. Cosine and projection mass are measured against the gradient-recovery shift at the same layer; the participation ratio (PR) and rank-$8$ explained variance measure how concentrated the per-example shift is.}
\label{tab:per_layer_pushforward}
\input{tables/table_per_layer.tex}
\end{table}

\paragraph{The activation shadow of the right perturbation aligns with the gradient.}
\label{sec:synthetic:activation}
As a foreshadowing of the LLM-scale steering results, we check the pushforward identity here. For trajectory-family weight perturbations, we push each one forward to activation space by equation~\eqref{eq:pushforward} and compare the per-layer activation shift to the shift produced by full gradient recovery (Table~\ref{tab:per_layer_pushforward}). The two are closely aligned at every layer --- cosine $\ge0.94$ from the embedding through the last block --- and the shift is concentrated in a low-dimensional subspace at each layer. For non-trajectory perturbations the same cosine is significantly lower. In other words, the activation shadow of a \emph{useful} weight move already points where the gradient points, and we will exploit this property below for steering.

\subsection{LoRA validation on DistilGPT-2 and GPT-2}
\label{sec:lora_validation}

The synthetic picture is clean and conforms to our expectations, but it is only one architecture with toy tasks. In this part, we repeat the probes on LoRA adapters fine-tuned over two real pretrained backbones: DistilGPT-2 with rank-$16$ adapters on the \task{sort}$\to$\task{mod-add} pair, and GPT-2 with rank-$8$ adapters on \task{reverse}$\to$\task{sort}, three seeds each, six runs in total. The recovery loop is the same. We will see that the qualitative conclusions transfer successfully, and one of them becomes our cleanest quantitative statement.

\paragraph{Low-rank geometry.}
\label{sec:lora:geometry}
Across the six runs, LoRA task-gradient covariances have effective rank between $3$ and $8$, and cross-task overlaps of a few percent, which is three to four orders of magnitude above random subspaces. A parameter-scope ablation (attention-only, MLP-only, or both) leaves this unchanged. The real backbone adapters reproduce the synthetic geometry faithfully.

\paragraph{The trajectory prefix captures the recovery, and that predicts search quality.}
\label{sec:lora:projection}
Here we get a clean empirical result about projection mass: how much of the recovery displacement $\Delta_{\rm GD}$ each candidate subspace contains (Figure~\ref{fig:projection_mass}a). A basis built from the first $10$ recovery steps captures about $77\%$ of $\Delta_{\rm GD}$; the first $20$ steps capture $\sim$$90\%$; the full trajectory captures all of it by construction. The static task planes capture far less, only about $15\%$ for the local plane and under $2\%$ for the pre-forgetting plane. This supports our hypothesis that the recoverable object follows the trajectory rather than stays on a fixed plane, and it holds on both backbones.

Importantly, projection mass combined with Proposition~\ref{prop:aligned_search} can {predict} how well random search will do inside each subspace. The per-dimension signal a matched-norm search extracts is $\norm{P_U\Delta_{\rm GD}}/\sqrt r$, so a subspace's mass and rank together predict its search ranking (Figure~\ref{fig:projection_mass}b); next, we verify that forecast.

\begin{figure}[!t]
\centering\setlength{\tabcolsep}{3pt}\footnotesize
\begin{tabular}{P{.48\linewidth}P{.48\linewidth}}
\multicolumn{2}{@{}c@{}}{\includegraphics[width=\linewidth]{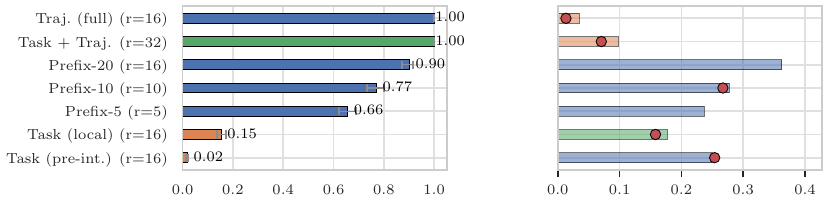}} \\[-2pt]
{(a) Projection mass $\norm{P_U \Delta_{\rm GD}}^2 / \norm{\Delta_{\rm GD}}^2$} & {(b) Predicted (bar) and observed (dot) signal density $\norm{P_U \Delta_{\rm GD}} / \sqrt r$} \\
\end{tabular}
\caption{Trajectory prefixes capture most of the LoRA recovery displacement, and that predicts search quality: (a)~how much of $\Delta_{\rm GD}$ each candidate subspace contains; (b)~the per-dimension signal density predicted by Proposition~\ref{prop:aligned_search} (bars) against observed best-of-$N$ gain (dots).}
\label{fig:projection_mass}
\end{figure}

\paragraph{Random search confirms the ordering, and the three scores disagree informatively.}
\label{sec:lora:randopt}
We now search inside each candidate subspace at matched step size on all six runs, and report all three scores (Table~\ref{tab:lora_search}). The ordering predicted by projection mass and density holds: the trajectory prefix wins, the full trajectory is close behind, and the static planes are far worse.

Interestingly, the three scores that we measure tell slightly different stories, and the differences are exactly what the theory would predict.
The \emph{Best-of-$N$} metric (used in Theorem~\ref{thm:gaussian_search}) follows the density prediction closely: the rank-$10$ prefix ($0.165$) beats the rank-$16$ full trajectory ($0.134$) by about the $\sqrt{16/10}$ factor the proposition predicts, and dividing each gain by $\sqrt r$ makes the ordering cleanly monotone. \emph{Ensemble-$k$-of-$N$} is more forgiving to the larger subspace (prefix $0.191$ vs.\ full trajectory $0.182$, nearly tied): pooling many candidates dilutes the per-dimension density penalty, plausibly because the extra trajectory directions are correlated rather than independent, so the full trajectory's effective rank is below its nominal $16$. \emph{Pass@$N$} simply tracks how much of $\Delta_{\rm GD}$ each subspace contains. We draw our main conclusion from best-of-$N$, where the theory applies most directly, and report the other two as informative context.

\begin{table}[!t]
\centering
\footnotesize
\caption{LoRA random search inside each candidate subspace at matched step size ($n=6$ runs). All numbers are mean target gains except pass@$N$.}
\label{tab:lora_search}
\begin{tabular}{lccccc}
\toprule
Candidate subspace & rank $r$ & best-of-$N$ & best-of-$N/\sqrt r$ & ens-$k$-of-$N$ & pass@$N$ \\
\midrule
Trajectory prefix ($10$ steps) & $10$ & $\mathbf{0.165}$ & $\mathbf{0.052}$ & $\mathbf{0.191}$ & $\mathbf{0.552}$ \\
Trajectory (full)              & $16$ & $0.134$ & $0.034$ & $0.182$ & $0.496$ \\
Task plane $+$ trajectory       & $32$ & $0.055$ & $0.010$ & $0.091$ & $0.396$ \\
Task plane (local)              & $16$ & $0.023$ & $0.006$ & $0.039$ & $0.270$ \\
\bottomrule
\end{tabular}
\end{table}

The conclusion we can draw from this is mostly about {where good weight moves live}.
Across these adapters, the directions that recover a forgotten task concentrate in a small subspace aligned with the early recovery trajectory, and \emph{not} in the static task plane. A random search confirms this by finding usable moves far more often inside the former than the latter. This is consistent with the moving-trajectory picture of Section~\ref{sec:theory:krylov}. We are cautious, however, about over-generalizing from two task pairs (see the ``Limitations'' section).

\paragraph{A closer look: successful perturbations mostly point where the gradient does.}
\label{sec:lora:scatter}
It helps to look below the family averages at individual candidates (Figure~\ref{fig:lora_signal_scatter}). Two things are visible. 

First, perturbations that improve the task are largely the ones aligned with the gradient descent recovery direction: a candidate's cosine with $\Delta_{\rm GD}$ correlates with its observed gain ($r\approx0.52$), and the search occasionally does a little better than a single gradient step by finding a more effective {scale} along roughly the same direction. This means that descent and search are mostly consistent.

Second, a candidate's overlap with the {static local task plane} does \emph{not} predict its success; if anything it anti-correlates ($r\approx-0.40$), because that plane has the recovery signal projected mostly out of it. 
Again, we see that useful moves live along the moving trajectory, and reading them off the static plane would be misleading. A separate, more artificial probe---forcing gradient descent to stay inside each candidate subspace---is reported in Appendix~\ref{app:projected}; we do not draw geometric conclusions from it, since constraining descent to a subspace the gradient has already left is expected to behave erratically.

\begin{figure}[!t]
\centering
\includegraphics[width=\linewidth]{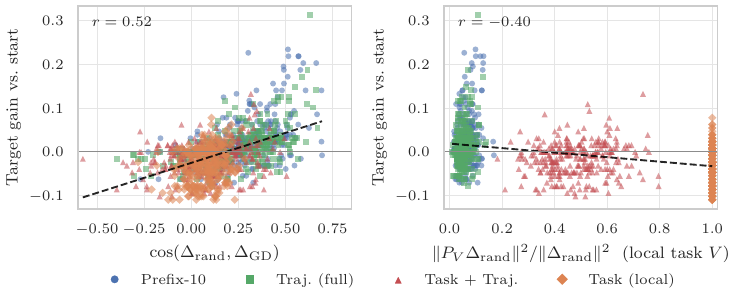}
\caption{{Per-candidate view ($\sim$$1.2$k candidates pooled over $6$ runs)}: (a)~alignment with the gradient recovery direction predicts gain ($r\approx0.52$);
(b)~overlap with the static local task plane does {not} predict gain ($r\approx-0.40$).}
\label{fig:lora_signal_scatter}
\end{figure}

\paragraph{A practical aside.}
\label{sec:lora:recipes}
The moving plane hypothesis makes a specific prediction about multitask fine-tuning. If you train one task to convergence and only then add another, the first task's local plane will have drifted by the time you come back, and re-recovering it has to traverse directions the original plane no longer contains, so \emph{sequential} specialisation should be brittle. Our experiments, shown in Figure~\ref{fig:lora_recipes} and Table~\ref{tab:lora_recipes_appendix}, confirm that it is: on both backbones, schedules that update the tasks \emph{together}---simple gradient-sum, alternating, or PCGrad-lite~\citep{yu2020gradient}---keep worst-task accuracy an order of magnitude higher than schedules that finish one task before starting the next. We are not proposing a new optimizer, and in our experiments simultaneous schedules cluster within a couple of points of each other with no systematic winner. But it is clear that sequential specialization should not be.

\begin{figure}[!t]
\centering
\centering\setlength{\tabcolsep}{3pt}\small
\begin{tabular}{P{.47\linewidth}P{.47\linewidth}}
\multicolumn{2}{c}{\includegraphics[width=\linewidth]{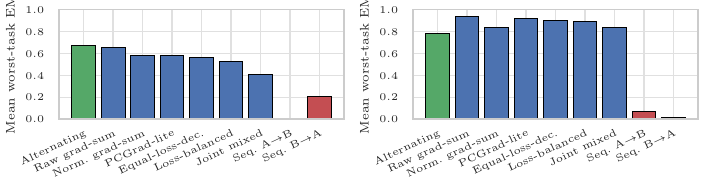}} \\
{ (a) DistilGPT-2 \task{sort}$\to$\task{mod-add}} & { (b) GPT-2 \task{reverse}$\to$\task{sort}} \\
\end{tabular}
\caption{Sequential specialization is brittle, simultaneous schedules are robust: mean worst-task exact-match across three seeds per backbone for nine multitask schedules.}
\label{fig:lora_recipes}
\end{figure}

\begin{table}[!t]
\centering\setlength{\tabcolsep}{3pt}\small
\caption{Mean worst-task EM, $n=3$ seeds.}
\label{tab:lora_recipes_appendix}
\begin{tabular}{lcc}
\toprule
\textbf{Schedule} & \shortstack{D-GPT2\\\task{sort}/\task{m-add}} & \shortstack{GPT2\\\task{rev}/\task{sort}} \\
\midrule
Raw grad-sum        & $0.656$          & $\mathbf{0.938}$ \\
Alternating         & $\mathbf{0.677}$ & $0.779$ \\
PCGrad-lite         & $0.583$          & $0.922$ \\
Equal-loss-dec.     & $0.565$          & $0.904$ \\
Norm.\ grad-sum     & $0.586$          & $0.841$ \\
Loss-balanced       & $0.531$          & $0.891$ \\
Joint mixed         & $0.406$          & $0.836$ \\
Seq.\ A$\to$B       & $0.000$          & $0.068$ \\
Seq.\ B$\to$A       & $0.208$          & $0.018$ \\
\bottomrule
\end{tabular}
\end{table}

\subsection{LLM scale}
\label{sec:steering}

The synthetic and LoRA results establish the picture up to $\sim$$10^6$ trainable parameters. Scaling to billion-parameter models, we ask three questions in turn:
does the low-rank gradient geometry survive; does budgeted random search actually find improvements, and in what regime; and does the weight-to-activation bridge produce usable steering vectors? Below, we give answers to these questions

\paragraph{The low-rank picture survives only partly.}
\label{sec:steering:lowrank}
On Qwen2.5-0.5B we estimate per-task gradient covariances through a random sketch and read off effective ranks; they come out small (single digits to low teens) and task-dependent. But a small effective rank from only a few dozen minibatches in a $D\gg n$ space can be a sampling artifact rather than a property of the task, so we run a control experiment: replace each gradient by a same-norm Gaussian vector and re-estimate the rank under the same pipeline. The result is mixed, and we report it in Figure~\ref{fig:qwen_gradient_rank}. On Qwen-3B with the MLP scope, the measured rank is clearly below the random-sketch control, but 
on Qwen-7B the measured and control ranks essentially coincide, and the apparent low rank is indistinguishable from sampling noise at this sample size.
So the clean low-rank result is established on the small and LoRA models, but only partially survives at billion-parameter scale.

\begin{figure}[!t]
\centering
\includegraphics[width=.8\linewidth]{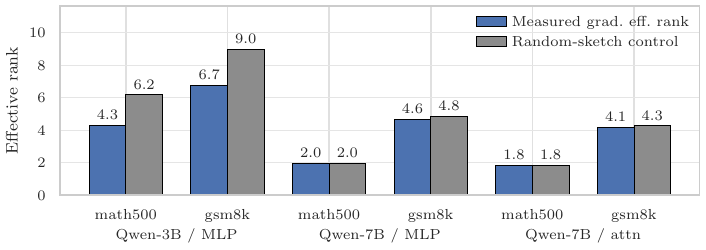}
\caption{{Measured gradient rank vs.\ a random-sketch control on Qwen.}} 
\label{fig:qwen_gradient_rank}
\end{figure}

\paragraph{Gaussian search finds improvements, in the regime the theorem predicts.}
\label{sec:steering:search}
Theorem~\ref{thm:gaussian_search} only applies where the loss is roughly linear in the perturbation. We locate that regime directly with a symmetric finite-difference probe that separates the first-order part of a perturbation's effect from the curvature part. We define 
$$
\begin{array}{ll}
F_1=\tfrac12[L(\theta+\sigma\delta)-L(\theta-\sigma\delta)] & \text{(antisymmetric, linear part),}\\
F_2=\tfrac12[L(\theta+\sigma\delta)+L(\theta-\sigma\delta)-2L(\theta)] & \text{(symmetric, curvature part),}
\end{array}
$$
and sweep the scale $\sigma$. On Qwen2.5-0.5B the linear part dominates around $\sigma\approx10^{-4}$, and 
above it curvature takes over and random moves become destructive (Figure~\ref{fig:f1f2_scatter}). The same sweep shows the regime is task-dependent: BoolQ and HellaSwag have a healthy band of linear and convex directions, while ARC-Easy is dominated by concave directions at every scale, which is the case where no fixed direction is preferred and random search degenerates into noise.

\begin{figure}[!t]
\centering
\includegraphics[width=\linewidth]{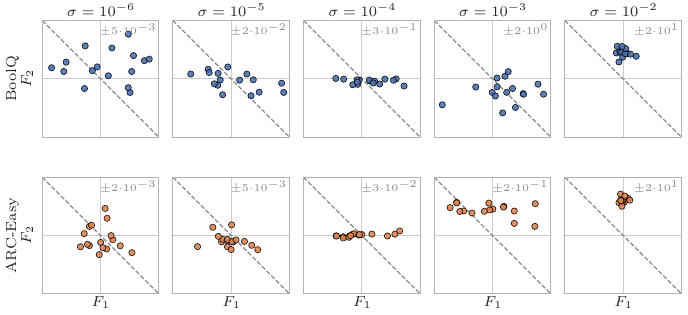}
\caption{{Locating the linear regime on Qwen2.5-0.5B.} Each panel shows the first-order ($F_1$) vs.\ curvature ($F_2$) decomposition of random perturbations at a given scale. The linear part dominates around $\sigma\approx10^{-4}$; at larger scales curvature dominates and every direction degrades the loss.}
\label{fig:f1f2_scatter}
\end{figure}

Inside the predicted window the theorem's result is confirmed: best-of-$N$ random parameter search at $\sigma\approx10^{-4}$ produces task-improving perturbations whose ensembled accuracy on a held-out BoolQ evaluation reaches roughly $+10$ points, with damage to the other tasks staying small.
 When the regime assumption fails (too small scale, or a curvature-dominated task like ARC-Easy), the same procedure yields nothing useful. The theorem and the scale window together indicate {when} random search around a large model should work, and the experiment agrees.

\paragraph{From a useful weight perturbation to a steering vector.}
\label{sec:steering:transport}
The pushforward identity~\eqref{eq:pushforward} says that a useful weight perturbation has an activation-space shadow that should itself steer behaviour. We test this with two experiments.

\begin{figure}[!t]
\centering
\includegraphics[width=\linewidth]{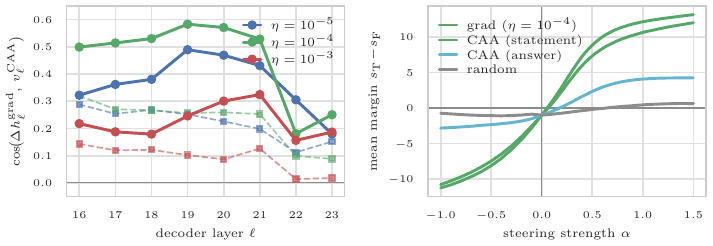}
\caption{{A single gradient step yields a CAA-like steering vector at late layers (Qwen2.5-0.5B)}: (a)~cosine between the gradient-induced activation shift and the contrastive CAA vector across layers;
(b)~steering response.}
\label{fig:gradient_steering}
\end{figure}

\emph{The clean version: one gradient step is a steering vector.}
First, we use the gradient itself. On Qwen2.5-0.5B with a true/false statement format, we take one gradient step, push the resulting weight change forward to each decoder layer's activations, and compare it to a contrastive (true-minus-false) CAA steering vector built independently from labelled prompts (the two routes are sketched in Figure~\ref{fig:clean_steering}). The two align well at late layers: the cosine rises through the mid-decoder block and peaks around $0.58$ at layers $19$--$20$, which are the same layers steering vectors are usually drawn from, and only at the intermediate learning rate $10^{-4}$ that the $F_1/F_2$ probe flagged as linear. At $10^{-3}$, the step leaves the linear regime and the alignment collapses. Reinjecting the gradient-induced shift into a clean model produces a steering response curve that closely tracks the CAA vector's and improves held-out accuracy by up to $+10$ points (Figure~\ref{fig:gradient_steering}). The model behaves nearly the same whether the steering vector came from $80$ labelled contrast pairs or from a single gradient step with no labels at all. This is one model and one task, and we do not claim the two vectors are identical (though they are linked by our identity), but the match is strong.

\begin{figure}[!t]
\centering
\begin{tikzpicture}[
  font=\small,
  >={Latex[length=2.4mm,width=1.8mm]},
  node distance=6mm,
  box/.style={rounded corners=3pt, draw, line width=0.7pt, align=center,
              inner sep=4pt, minimum height=15mm},
  blue/.style ={box, fill=cblue,  draw=dblue},
  amber/.style={box, fill=camber, draw=damber},
  green/.style={box, fill=cgreen, draw=dgreen},
  flow/.style={->, line width=0.9pt, draw=black!75},
]
\def\colA{0}    \def\colB{4.3}   \def\colC{8.9}
\def\rowT{0}    \def\rowM{-2.2}   \def\rowMid{-1.1}

\node[blue, text width=27mm] (grad) at (\colA,\rowT)
      {\textbf{Gradient step}\\[2pt]$\delta\theta=-\eta\nabla L$};
\node[blue, text width=41mm] (push) at (\colB,\rowT)
      {\textbf{Pushforward}\\[2pt]$v_\ell^{\mathrm{grad}}=\E_x[D_\theta h_\ell\,\delta\theta]$};

\node[amber, text width=27mm] (contrast) at (\colA,\rowM)
      {\textbf{Labelled}\\\textbf{true/false prompts}};
\node[amber, text width=41mm] (caa) at (\colB,\rowM)
      {\textbf{CAA vector}\\[2pt]$v_\ell^{\mathrm{CAA}}=\bar h^{+}-\bar h^{-}$};

\node[green, text width=33mm, minimum height=27mm] (cmp) at (\colC,\rowMid)
      {\textbf{Compare \& inject}\\[3pt]cosine of $v_\ell^{\mathrm{grad}},v_\ell^{\mathrm{CAA}}$;\\inject $\alpha v_\ell$ at layer $\ell$;\\measure response \& accuracy};

\draw[flow] (grad) -- (push);
\draw[flow] (contrast) -- (caa);
\draw[flow] (push.east) -- (cmp.west |- push.east);
\draw[flow] (caa.east)  -- (cmp.west |- caa.east);
\end{tikzpicture}
\caption{{The clean (deterministic) route to a steering vector.} Two independent constructions --- a single gradient step pushed forward to activations, and a labelled-contrast CAA vector $\bar h^{+}-\bar h^{-}$ --- are compared by cosine and injected into the unperturbed model. They agree at the late decoder layers; the resulting alignment and steering response are reported in Figure~\ref{fig:gradient_steering}. Compare with Fig.~\ref{fig:steering_pipeline}.}
\label{fig:clean_steering}
\end{figure}
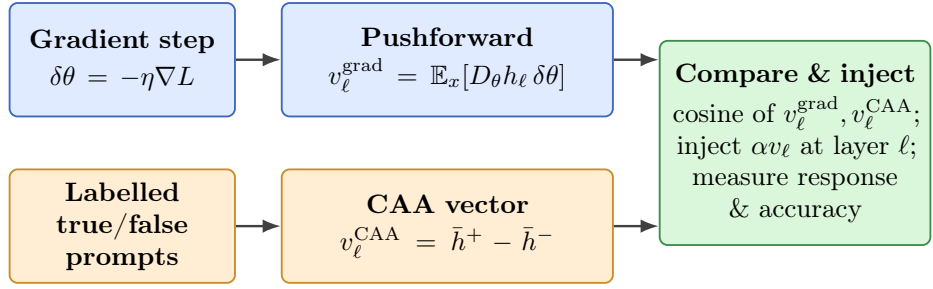

\begin{figure}[!t]
\centering
\begin{tikzpicture}[
  font=\small,
  >={Latex[length=2.4mm,width=1.8mm]},
  node distance=6mm,
  box/.style={rounded corners=3pt, draw, line width=0.7pt, align=center,
              inner sep=4pt, minimum height=15mm},
  blue/.style ={box, fill=cblue,  draw=dblue},
  amber/.style={box, fill=camber, draw=damber},
  green/.style={box, fill=cgreen, draw=dgreen},
  purple/.style={box, fill=cpurple, draw=dpurple},
  flow/.style={->, line width=0.9pt, draw=black!75},
  alabel/.style={font=\footnotesize, inner sep=1.5pt, fill=white, rounded corners=1pt},
]

\def\colA{0}    \def\colB{3.8}   \def\colC{7.6}   \def\colD{12.3}
\def\rowT{0}    \def\rowM{-2.2}   \def\rowB{-4.4}
\def\rowMid{-1.1}  

\node[blue,  text width=26mm] (ssel) at (\colA,\rowT)
      {\textbf{Search split}\\[1pt]$\mathcal{S}_{\mathrm{sel}}$};
\node[amber, text width=26mm] (svec) at (\colA,\rowM)
      {\textbf{Hidden-state split}\\[1pt]$\mathcal{S}_{\mathrm{vec}}$};
\node[green, text width=26mm] (srep) at (\colA,\rowB)
      {\textbf{Report split}\\[1pt]$\mathcal{S}_{\mathrm{rep}}$ \footnotesize(held out)};

\node[blue,  text width=37mm] (cand) at (\colB,\rowT)
      {\textbf{Candidate\\ perturbations}\\[2pt]sample $\delta\sim\mathcal{N}(0,\sigma^2 I)$;\\keep best-of-$N$ on $\mathcal{S}_{\mathrm{sel}}$};
\node[amber, text width=37mm] (steer) at (\colB,\rowM)
      {\textbf{Steering vector}\\[2pt]average activation shift\\$v_\ell=\overline{\Delta h_\ell}$ of kept $\delta$};

\node[purple, text width=21mm, minimum height=27mm] (sel) at (\colC,\rowMid)
      {\textbf{Operating-point}\\\textbf{selection}\\[3pt]sweep $\ell,\alpha$;\\score on $\mathcal{S}_{\mathrm{sel}}$};

\node[green, text width=29mm, minimum height=27mm] (fin) at (\colD,\rowMid)
      {\textbf{Final report}\\[3pt]inject $\alpha^{\star}v^{\star}$ at $\ell^{\star}$;\\score on $\mathcal{S}_{\mathrm{rep}}$};

\draw[flow] (ssel) -- (cand);
\draw[flow] (svec) -- (steer);
\draw[flow] (cand) -- (steer);                                  
\draw[flow] (steer.east) -- (sel.west |- steer.east);
\draw[flow] (sel) -- node[alabel, above=0.3mm] {$(\ell^{\star},\alpha^{\star},v^{\star})$} (fin);

\draw[flow] (srep.east) -| ($(fin.south)+(0,-0.0)$);

\end{tikzpicture}
\caption{{The perturbation-to-steering pipeline (random-search version).} The three splits are disjoint subsets of the evaluation examples: on the search split $\mathcal{S}_{\mathrm{sel}}$ we keep the best-of-$N$ random weight perturbations; on the hidden-state split $\mathcal{S}_{\mathrm{vec}}$ we average their activation shadow into a steering vector; the operating point (layer, scale) is chosen on $\mathcal{S}_{\mathrm{sel}}$; and the held-out report split $\mathcal{S}_{\mathrm{rep}}$ is used only for the final numbers. Compare with Fig.~\ref{fig:clean_steering}.}
\label{fig:steering_pipeline}
\end{figure}
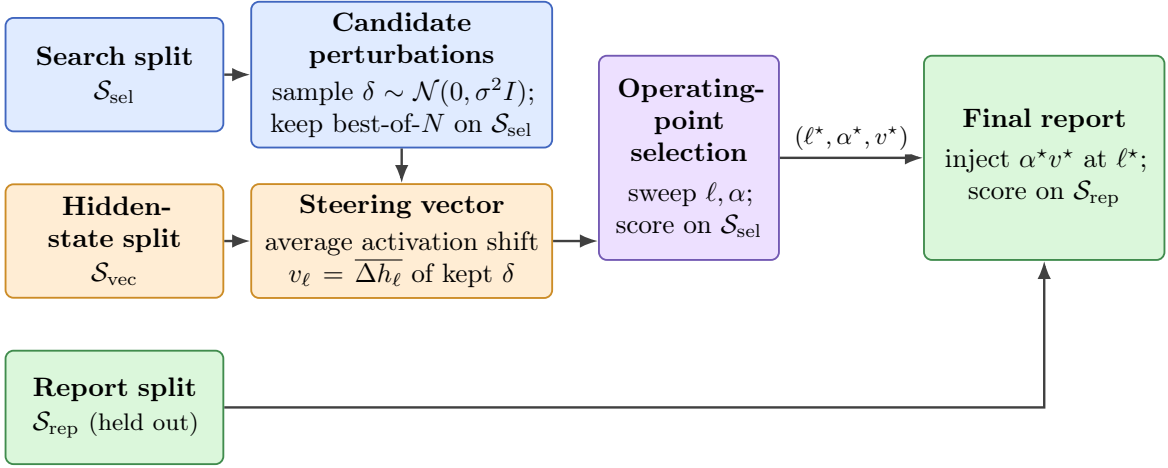

\emph{The noisy version: selected random perturbations, across models.}
The random-search counterpart replaces the gradient step with a best-of-$N$ selected Gaussian weight perturbation. The pipeline (Figure~\ref{fig:steering_pipeline}) uses three disjoint subsets of the evaluation examples: on the \emph{search split} $\mathcal{S}_{\mathrm{sel}}$ we sample weight perturbations $\delta\sim\mathcal{N}(0,\sigma^2 I)$ and keep the best-of-$N$ that improve the target; on the \emph{hidden-state split} $\mathcal{S}_{\mathrm{vec}}$ we average the activation shifts those kept perturbations induce into a steering vector $v_\ell$; we then choose the operating point --- the decoder layer (swept over the mid-to-late layers, around $\ell\in\{15,\dots,20\}$) and the injection scale $\alpha$ --- on $\mathcal{S}_{\mathrm{sel}}$; and we report on the held-out \emph{report split} $\mathcal{S}_{\mathrm{rep}}$. We run this across $14$ (model, task) variants spanning Qwen2.5 at $0.5$B / $3$B / $7$B and OLMo-7B (full numbers in Table~\ref{tab:multimodel_main}). As a control we compare against \emph{matched random steering}: a random hidden-state direction injected at the \emph{same layer} and rescaled to the \emph{same norm} as the structured vector --- a like-for-like steering baseline, not a test of the theorem (which concerns weight-space search rather than activation injection). Structured transport beats this control on $9$ of the $14$ variants, and the Qwen BoolQ effect survives at all three scales. It is not free or universal; in particular, the OLMo BoolQ lane buys a $+1.2$-point target gain at the cost of $-17$ points on the other tasks. In summary, the bridge is real and often useful, but task- and architecture-dependent.

\begin{table}[!t]
\centering
\small
\caption{Multimodel perturbation-to-steering ($n=14$ lanes). \emph{Base} is the unperturbed report-split accuracy; all other columns are accuracy deltas (points) on the held-out report split. \emph{Param best} is the best target gain from a kept weight perturbation itself (the candidate, before transport to a steering vector); \emph{Steer best} is the best gain from the steering vector built from those perturbations, at the operating point chosen on the selection split; \emph{Side mean / worst} are the average and worst side-task deltas there; \emph{Random best} is the best gain from a random hidden-state direction injected at the same layer and matched to the same norm (a like-for-like control).}
\label{tab:multimodel_main}
\setlength{\tabcolsep}{4pt}
\begin{tabular}{llccccc}
\toprule
Model & Task & Base & Param best & Steer best & Side mean / worst & Random best \\
\midrule
Qwen2.5-0.5B & BoolQ      & $59.0$ & $+4.7$ & $\mathbf{+10.7}$ & $-4.7$ / $-8.4$  & $+5.7$ \\
Qwen2.5-0.5B & ARC-Easy   & $69.3$ & $+1.6$ & $+4.1$           & $-2.0$ / $-8.6$  & $+0.6$ \\
Qwen2.5-0.5B & HellaSwag  & $42.0$ & $+2.3$ & $+2.7$           & $-2.6$ / $-15.8$ & $-6.6$ \\
Qwen2.5-3B   & BoolQ (str.)& $72.3$ & $+3.1$ & $\mathbf{+6.2}$  & $-0.2$ / $-0.4$  & $+0.0$ \\
Qwen2.5-3B   & BoolQ (lite)& $72.7$ & $+0.0$ & $+5.5$           & $-0.1$ / $-0.4$  & $+0.8$ \\
Qwen2.5-3B   & ARC-Easy   & $92.6$ & $+0.0$ & $+0.8$           & $+0.9$ / $-0.4$  & $-1.2$ \\
Qwen2.5-3B   & HellaSwag  & $68.8$ & $+0.0$ & $+1.6$           & $+0.1$ / $+0.0$  & $+0.4$ \\
Qwen2.5-7B   & BoolQ (lite)& $84.8$ & $+0.0$ & $+3.1$          & $-1.0$ / $-1.2$  & $+0.0$ \\
Qwen2.5-7B   & BoolQ (med.)& $84.9$ & $+0.0$ & $+2.1$          & $-0.4$ / $-0.5$  & $+0.3$ \\
Qwen2.5-7B   & ARC-Easy   & $96.9$ & $+1.6$ & $+0.0$           & $+0.6$ / $-0.4$  & $-0.8$ \\
Qwen2.5-7B   & HellaSwag  & $73.0$ & $+1.6$ & $\mathbf{+3.1}$  & $-0.2$ / $-0.8$  & $-0.8$ \\
OLMo-7B      & BoolQ      & $77.0$ & $+2.1$ & $+1.2$           & $\mathbf{-17.0}$ / $-21.1$ & $-3.9$ \\
OLMo-7B      & ARC-Easy   & $79.3$ & $+0.0$ & $+2.3$           & $+0.4$ / $-0.4$  & $+0.0$ \\
OLMo-7B      & HellaSwag  & $53.1$ & $+2.1$ & $+0.4$           & $+0.0$ / $+0.0$  & $+0.0$ \\
\bottomrule
\end{tabular}
\end{table}

\section{Discussion and Limitations}
\label{sec:discussion}

\paragraph{Empirical conclusions.}
Our final conclusion is neither ``there is no linear structure'' nor ``there is a single global task vector''. Trained tasks induce \emph{local} low-dimensional structure whose meaning depends on where and how you look. At a single checkpoint, task gradients concentrate in a small subspace that is a good basis for a single descent step. Along a trajectory, the recovery displacement lives in a Krylov subspace well approximated by the early trajectory, and that is where a finite random search should look. Under random search, useful directions live in a scale window between too-small (signal below noise) and too-large (curvature dominates), and inside that window the gains are dimension-independent and focused on the target. In activation space, the shadow of a single gradient step closely tracks a contrastive steering vector at late layers, at least in the cases we tested. The unifying statement is that linear structure in trained networks is real but local: it moves as training proceeds, it is scale-dependent, and it looks different in weight space and activation space.

\paragraph{Limitations.}
The most important limitations of our study as as follows.
\begin{enumerate}
\item \emph{Scale of the evidence}: the cleanest results are on a $\sim$$5\!\cdot\!10^5$-parameter synthetic Transformer and on small LoRA adapters.
\item \emph{The LLM low-rank claim is partial}: at billion-parameter scale, a low effective rank separates from a sampling-noise control only in some settings (Qwen-3B / MLP), not in all; a larger, better controlled replication is a natural future work direction.
\item \emph{Few task pairs}: the trajectory-vs-plane comparison rests on one synthetic family and two LoRA task pairs; the ordering is consistent across them, but we are cautious about generalizing the exact magnitudes.
\item \emph{The trajectory prefix needs a trajectory}: while it shows where good moves live, we cannot turn it into an algorithm that would help in practice because building it requires already running the recovery.
\end{enumerate}

\section{Conclusion}
\label{sec:conclusion}

Modern language models exhibit many interesting phenomena tied to the internal geometry of their representations and their optimization dynamics. Understanding these structures matters for model safety, controllable adaptation, interpretability, and efficient training. Recent work has highlighted the surprising effectiveness of linear methods acting either in weight space or in activation space, and several explanations for these phenomena have been proposed, from mechanistic feature circuits to overcomplete latent representations motivating sparse autoencoders. In this work, we instead adopt a differential-geometric perspective that, in retrospect, appears natural: we treat small weight perturbations as infinitesimal displacements on the parameter manifold and study the local tangent structure induced by optimization dynamics.

From this viewpoint, many observed linear phenomena arise naturally. Neural networks are differentiable functions, and gradient-based training follows trajectories constrained by local geometry. Our results show that these trajectories are genuinely low-dimensional but not static: adaptation proceeds through evolving Krylov-like subspaces rather than fixed global task directions. This explains why local low-rank structure is simultaneously real, useful, and unstable under continued training. At a checkpoint, linear structures appear as local task-gradient subspaces; along training they form Krylov-like trajectory bundles; under random search they are scale-windowed Gaussian neighbourhoods; in activations they are the parameter-to-activation pushforward of those same gradient directions, exposed cleanly by a single gradient step at the linear-regime learning rate and noisily by best-of-$N$ random search.

The same perspective explains why random perturbation methods can work remarkably well even at LLM scale. In a locally linear regime, Gaussian random search efficiently identifies task-improving directions while producing diffuse side effects on orthogonal tasks. Our experiments further show that these local regimes do exist in practice, although their structure depends strongly on the task and the perturbation scale. We leave for future work a deeper question that this work raises but does not answer: why do fully trained checkpoints contain nearby perturbations capable of dramatically improving specific tasks at all? Understanding the origin of these highly task-selective local directions and when they emerge---during pretraining, instruction tuning, or as a property of the loss landscape itself---is an important direction.

Another contribution of this work is a unified view of several apparently different local adaptation methods: gradient updates, random parameter perturbations, and activation steering are connected through the local differential structure of the network. Successful parameter perturbations can induce reusable activation-space directions, and gradient-induced activation shifts can closely track contrastive activation-steering vectors in the locally linear regime, as we observed for Qwen-$0.5$B on BoolQ. A natural next step is to move beyond local tangent geometry towards a Lie-theoretic description of adaptation dynamics, with the geometry of local updates constrained by an underlying transformation group. We believe that connecting optimization, representation geometry, and local symmetry structure may yield a more unified theory of adaptation in large neural networks.

\bibliographystyle{plainnat}
\bibliography{references}

\clearpage

\appendix

\section{Synthetic transformer: additional results}
\label{app:synthetic}

The per-layer activation pushforward (Table~\ref{tab:per_layer_pushforward} in the main text) is the basis for the activation-shadow result. Here we note that the two alternative scorings of the synthetic random search reproduce the best-of-$N$ ordering of Table~\ref{tab:synth_search}. Scoring by the best \emph{iterative} (multi-step) recovery rather than the best single one-shot perturbation gives mean gains of $0.083$ (trajectory prefix), $0.046$ (local task plane), $0.022$ (pre-interference plane), and $\le 0.001$ for the isotropic and random controls --- the same ordering, with the trajectory prefix on top. The ensemble scoring behaves identically.

\section{LoRA: full search table and per-run breakdown}
\label{app:lora_full}

Table~\ref{tab:lora_perrun_appendix} gives the per-run ensemble and pass@$N$ numbers behind the aggregates of Table~\ref{tab:lora_search}. The family ordering (trajectory prefix or full trajectory on top, static planes at the bottom) reproduces in every one of the six runs; per run, the prefix and the full trajectory trade the top spot, which is why the aggregate ``prefix is best'' statement comes from averaging rather than from a clean sweep.

\begin{table}[!t]
\centering
\footnotesize
\caption{Per-run LoRA random search, ens-$k$-of-$N$ and pass@$N$ at $K=256$, $k=50$.}
\label{tab:lora_perrun_appendix}
\begin{tabular}{llcc}
\toprule
Run & Subspace & ens-$k$-of-$N$ EM & pass@$N$ \\
\midrule
DistilGPT-2 \task{sort}$\to$\task{mod-add}, s7  & Trajectory prefix  & $\mathbf{0.328}$ & $\mathbf{0.781}$ \\
DistilGPT-2 \task{sort}$\to$\task{mod-add}, s7  & Trajectory (full)  & $0.297$ & $0.688$ \\
DistilGPT-2 \task{sort}$\to$\task{mod-add}, s7  & Task plane (local) & $0.156$ & $0.383$ \\
DistilGPT-2 \task{sort}$\to$\task{mod-add}, s17 & Trajectory (full)  & $0.172$ & $0.414$ \\
DistilGPT-2 \task{sort}$\to$\task{mod-add}, s17 & Trajectory prefix  & $0.094$ & $0.398$ \\
DistilGPT-2 \task{sort}$\to$\task{mod-add}, s27 & Trajectory (full)  & $0.195$ & $0.523$ \\
DistilGPT-2 \task{sort}$\to$\task{mod-add}, s27 & Trajectory prefix  & $0.148$ & $0.531$ \\
GPT-2 \task{reverse}$\to$\task{sort}, s7        & Trajectory prefix  & $0.258$ & $0.531$ \\
GPT-2 \task{reverse}$\to$\task{sort}, s17       & Trajectory prefix  & $0.305$ & $0.516$ \\
GPT-2 \task{reverse}$\to$\task{sort}, s27       & Trajectory (full)  & $\mathbf{0.422}$ & $0.555$ \\
\bottomrule
\end{tabular}
\end{table}

\section{Projected gradient descent inside candidate subspaces}
\label{app:projected}

For completeness we report the more artificial probe alluded to in Section~\ref{sec:lora_validation}: forcing gradient descent to stay inside each candidate subspace for $50$ steps (Table~\ref{tab:projected}). The result is the mirror image of random search: the local task plane nearly matches unconstrained Full LoRA, while every trajectory prefix is weak --- and, notably, the different prefix lengths ($5$, $10$, $20$, full) are all about equally weak ($\sim$$0.012$), so there is no systematic ``more steps is worse'' effect, just a uniformly low value. We include this table only to document that constraining descent to a subspace the gradient has already drifted out of behaves erratically; we do \emph{not} use it to make geometric claims, since projected gradient descent is not a method anyone deploys and the constrained dynamics are not informative about where good \emph{finite} moves live.

\begin{table}[h]
\centering
\footnotesize
\caption{Projected gradient recovery for $50$ steps inside each candidate subspace ($n=6$). Mean gain over the forgotten-checkpoint baseline.}
\label{tab:projected}
\begin{tabular}{lccc}
\toprule
Subspace / method & mean gain & mean final EM & mass on $\Delta_{\rm GD}$ \\
\midrule
Full LoRA (unconstrained)   & $\mathbf{0.060}$ & $\mathbf{0.098}$ & $1.000$ \\
Task plane (local)          & $0.052$ & $0.090$ & $0.147$ \\
Trajectory prefix ($20$)    & $0.012$ & $0.049$ & $0.872$ \\
Trajectory prefix ($10$)    & $0.012$ & $0.049$ & $0.753$ \\
Trajectory (full)           & $0.012$ & $0.049$ & $1.000$ \\
Trajectory prefix ($5$)     & $0.010$ & $0.048$ & $0.622$ \\
Task plane (pre-int.)       & $0.009$ & $0.047$ & $0.023$ \\
Random rank-matched         & $0.000$ & $0.038$ & $0.000$ \\
\bottomrule
\end{tabular}
\end{table}

\section{Explicit Krylov basis as a curvature diagnostic}
\label{app:krylov}

We computed the explicit Krylov basis $\Span\{g,Hg,H^2g,H^3g\}$ at $\theta_{\rm forg}$ via symmetric finite-difference Hessian-vector products. Across the six runs it captures only $\sim$$11\%$ of the recovery displacement in projection mass, against $\sim$$77\%$ for the empirical $10$-step trajectory prefix. Its random-search ranking is above the gradient-alone and random controls but well below the trajectory prefix. Consistent with our remark to Lemma~\ref{lem:krylov}, the single-Hessian basis is a useful curvature \emph{diagnostic} but not a reconstruction of the recovery path, which the directly observed trajectory captures far better. We therefore use the trajectory prefix, not the explicit Krylov basis, as our working proxy throughout.

\end{document}

%% file: tables/table_per_layer.tex
\begin{tabular}{lcccc}
\toprule
Layer & $\cos(\delta h_\ell^{\rm evol}, \delta h_\ell^{\rm GD})$ & $\|P_{\rm GD}\delta h_\ell^{\rm evol}\|^2 / \|\delta h_\ell^{\rm evol}\|^2$ & Shift PR & Rank-8 EV (\%) \\
\midrule
Embedding & 0.941 & 0.886 & 1.7 & 98.4 \\
Block 0 & 0.970 & 0.941 & 2.9 & 94.3 \\
Block 1 & 0.962 & 0.925 & 7.2 & 82.2 \\
Block 2 & 0.975 & 0.951 & 9.6 & 76.1 \\
Block 3 & 0.985 & 0.970 & 2.4 & 88.3 \\
Final norm & 0.877 & 0.769 & 11.4 & 72.4 \\
\bottomrule
\end{tabular}